\documentclass{svproc}
\usepackage{amsmath,amssymb,amsfonts}
\usepackage{graphicx}
\usepackage{textcomp}
\usepackage{graphicx,subfigure}

\usepackage{algorithm}
\usepackage{algorithmic}

\usepackage{listings}
\usepackage{xcolor}
\usepackage{matlab-prettifier}
\usepackage{tabularx}
\usepackage{graphicx}
\graphicspath{ {./images/} }
\usepackage{accents}
\usepackage[english]{babel}
\usepackage{mathtools}
\mathtoolsset{showonlyrefs}
\usepackage{comment}

\usepackage{url}

\begin{document}
\mainmatter              
\title{Control Theoretic Approach to Fine-Tuning and Transfer Learning}
\titlerunning{Control Theoretic Approach to Fine-Tuning and Transfer Learning}  
%
\author{Erkan Bayram\inst{1} \and Shenyu Liu\inst{2} \and
Mohamed-Ali Belabbas\inst{1} \and Tamer Başar\inst{1}}
\authorrunning{Erkan Bayram et al.} 
%
\tocauthor{Erkan Bayram, Shenyu Liu, Mohamed-Ali Belabbas, Tamer Başar}
\institute{Coordinated Science Laboratory, University of Illinois at Urbana-Champaign\\ Urbana, IL 61801\\
\email{(ebayram2,belabbas,basar1)@illinois.edu},\\ 
\and
The School of Automation, Beijing Institute of  Technology, Beijing 100089, China\\
\email{shenyuliu@bit.edu.cn}}

\maketitle              

\begin{abstract}
Given a training set in the form of a paired $(\mathcal{X},\mathcal{Y})$, we say that the control system $\dot x = f(x,u)$ has learned the paired set via the control $u^*$ if the system steers each point of $\mathcal{X}$ to its corresponding target in $\mathcal{Y}$. If the training set is expanded, most existing methods for finding a new control $u^*$ require starting from scratch, resulting in a quadratic increase in complexity with the number of points. To overcome this limitation, we introduce the concept of {\em tuning without forgetting}. We develop {\em an iterative algorithm} to tune the control $u^*$ when the training set expands, whereby points already in the paired set are still matched, and new training samples are learned. At each update of our method, the control $u^*$ is projected onto the kernel of the end-point mapping generated by the controlled dynamics at the learned samples. It ensures keeping the end-points for the previously learned samples constant while iteratively learning additional samples.

\end{abstract}
\section{Introduction}
\label{sec:introduction}
In recent years, neural networks and other data-driven architectures have seen a tremendous growth both in their abilities and in the size of the models they necessitate. Indeed, at the present, models with billions of parameters are not uncommon. Due to their sheer size, the amount of computation necessary to train such models is a major hurdle to their general adoption. This issue is compounded by the fact that one may want a model to be trained for an additional task or refined using newly acquired data. To avoid having to retrain the model, and the computation cost and time that ensue, methods for {\em transfer learning} and {\em fine tuning} are gaining rapid adoption. In a nutshell, these are methods that allow for updating of the parameters of a model trained on a given task to perform a new task (for transfer learning) and methods to update the parameters of a model when more examples of a given task are received (for fine-tuning). In this paper, we show that control theory provides a natural framework, as well as efficient algorithms, to perform these tasks.
Mathematically, fine-tuning and transfer learning are often treated similarly, as these methods are blind to the nature of the examples they are fed~\cite{gouk2021distance}. This also applies to the algorithms we propose here, but we focus on fine-tuning for the sake of exposition.


The models we consider are controlled differential equations of the type
\begin{align}\label{eqn:model_intro}
    \dot x = f(x,u) \\[-2em]
\end{align}
(see Section~\ref{sec:prelim} for their use in learning and the proper definitions). For a given paired training set $(\mathcal{X},\mathcal{Y})$ with finite cardinality $q$, the learning task is to find a control $u^*$ so that the flow $\varphi_t(u^*,\cdot)$ generated by~\eqref{eqn:model_intro} satisfies 
\begin{align}\\[-2em]
   R(\varphi_T(u^*,x^i))=y^i, \mbox{ for all } (x^i,y^i)\in (\mathcal{X},\mathcal{Y})\\[-2em]
\end{align}
for some finite $T\geq0$ and a given fixed readout map $R$. This task  is also known as the control of an ensemble of points~\cite{agrachev2020control}. Thus, we interchangeably refer to the training sample set $\mathcal{X}$ as the input ensemble and the training label set $\mathcal{Y}$ as the output ensemble in line with the usage in the control theory literature.  

In this work, we address the need to expand the training set $\mathcal{X}$ by introducing new samples. Our goal is to effectively adapt a control function $u^*$ trained on the original set to the expanded training set. This challenge aligns with the main goal of {\em fine-tuning} in the learning literature, a concept where one aims to adapt parameters $u^*$ trained for a given problem by further training them on another dataset. 

One of the main issues addressed by fine-tuning is to avoid a loss of performance on previous tasks while specializing the model to a new dataset. This problem is also known as {\em catastrophic forgetting} in the continual learning (incremental learning) literature.
Common approaches involve learning a sequence of tasks one by one and aiming to match the performance as if they were observed simultaneously~\cite{wang2024comprehensive}. However, in continual learning, access to previously observed data points is typically not available. When such access exists, it is usually discussed in the literature on robust fine-tuning.  Similar approaches to our work can be found in optimization-based continual learning methods that incorporate memory~\cite{lopez2017gradient}.

In the control theory literature, finding a control $u$ that steers each point in $\mathcal{X}$ with the cardinality $q$ to the associated labels in $\mathcal{Y}$ is done by concatenating all points in $\mathcal{X}$ in a state of higher dimension, and replicate the controlled dynamics $q$ times~\cite{agrachev2020control}. Then, this approach requires solving a controlled ordinary differential equation (CODE) for the $q$-tuple system. We call this method the {\em $q$-folded method}, details of which will be discussed later in the paper. This approach has two practical shortcomings: 
\begin{itemize}
    \item \textbf{Tuning Ability:} If the training set is augmented with additional samples, it is necessary to apply the $q$-folded method on the new set from scratch.
    \item \textbf{Scalability:} The $q$-folded method has complexity of $\mathcal{O}(  n^2 q^2 N )$ per iteration with $N$ time-step discretization. Therefore, as the size of the training set, $q$, grows, the complexity of the $q$-folded method increases quadratically, leading to scalability issues. 
\end{itemize}

We introduce here a novel fine-tuning method. It is based on the hereby introduced concept of {\em tuning without forgetting ${\mathcal{X}}$}. The method consists of an iterative algorithm in which each update on the control function $u$, denoted by $\delta u$, is selected as a projection of the gradient of the given cost functional onto the kernel of end-point mapping generated by the controlled dynamics. It ensures keeping the end-points for the previously steered points constant up to the first order while simultaneously steering the new points of the training set to their associated labels.

Our work thus contributes to the tuning ability and scalability of the methods using controlled dynamical systems for supervised learning and addresses the challenges stemming from training set updates. We will compare our proposed method to an existing fine-tuning method in a computational example.

Furthermore, our method allows to iteratively learn the samples in $(\mathcal{X},\mathcal{Y})$, thus bypassing the need to rely on dealing with a control system of dimension scaling with $q$. This is particularly important since, as we mentioned earlier, this scaling causes computational issues, but it can additionally run into memory capacity issues as well.  These features make our approach a preferred alternative to the existing $q$-folded method in the control of an ensemble of points.

\textbf{Related Works:} A popular idea is to restrict the distance between the updated parameters and the original parameters~\cite{gouk2021distance} to avoid a loss of performance on the previous task while specializing the model to a new dataset. The works~\cite{xuhong2018explicit,li2018delta} propose to use Euclidean norm and the matrix absolute row sum matrix norm (MARS) regularization  to restrict the distance of tuned parameters from the reference parameters. The work~\cite{tian2023trainable} proposes a trainable projected gradient method (TPGM) for fine-tuning to automatically learn the distance constraints for each layer.

Other optimization-based approaches in the continual learning literature include~\cite{lopez2017gradient}, which proposes the so-called ``gradient episodic memory (GEM)'', that performs an update in weight in the direction of gradient of loss function if it has negative inner product with the gradient of loss for the previous samples. Furthermore, the works~\cite{farajtabar2020orthogonal} proposes Orthogonal Gradient Descent (OGD) that stores the gradients for the initially observed points and rectifies the current gradient to be orthogonal to them. The study~\cite{wang2021training} introduces an algorithm called Adam-NSCL, which sequentially optimizes network parameters in the null space of previous tasks. This null space is approximated by applying singular value decomposition to the covariance matrix of all input features of previous tasks for each linear layer.

The dynamical system approach, which we adopt here, has been used in the past in the learning literature. For example, the studies~\cite{chen2018neural,dupont2019augmented,lu2018beyond} consider a continuous flow and introduce neural ordinary differential equations (NODE). The works presented in~\cite{cuchiero2020deep,tabuada2022universal} focus on the universal approximation capabilities of deep residual neural networks from a geometric control viewpoint. Our work falls within the scope of these two pieces of literature.

\section{Preliminaries}\label{sec:prelim}
In this section, we introduce the problem and provide the preliminaries for the iterative algorithm. Consider the paired sets $(\mathcal{X},\mathcal{Y})=\{(x^i, y^i)\}_{i=1}^q$, where the $x^i \in \mathbb{R}^{n}$'s are called initial points so that they are in a connected Riemannian manifold $\mathcal{M}$.  and the associated $y^i \in \mathbb{R}^{n_o}$'s are target points (labels). We assume that the elements of the input ensemble (training sample set) $\mathcal{X}$ are pairwise distinct, i.e., $x^i \neq x^j$ for $i \neq j$ and  Let $\mathcal{I}$ be an index set labelling the entries of the ensembles. 
We take the system:
\begin{equation}\label{eqn:control_system}
    \Dot{x}(t) = f( u(t) x(t) ) , u \in L_{\infty}([0,T], \mathbb{R}^{\bar{n}\times \bar{n}})
\end{equation}
where $x(t)\in\mathbb{R}^{\bar{n}}$ is the state vector at time $t$ and $f(\cdot)$ is smooth vector fields on $\mathbb{R}^{\Bar{n}}$. For example, if we select $f(\cdot) = \tanh(\cdot)$, the function is applied elementwise on the product $u(t)x(t)$, resulting in $\dot{x}(t) = \tanh( u(t) x(t) )$. Let
$E:\mathbb{R}^{n} \to \mathbb{R}^{\Bar{n}}$ be an injective function with $\bar{n}\geq n$, called the uplift function, which is independent of the control $u$. Let $R:\mathbb{R}^{\bar{n}} \to \mathbb{R}^{n_o}$ be a given function, called the {\em readout map}, which is independent of the control $u$ and so that the jacobian of $R$ is of full row rank. We next define memorization property for a dynamical system.

\begin{definition}[Memorization Property] Assume that a paired set $(\mathcal{X},\mathcal{Y})$, a fixed readout map $R$ and an up-lift function $E$ are given. The control $u$ is said to have {\em memorized the ensemble $(\mathcal{X},\mathcal{Y})$} for the model $\dot{x}(t)=$ $f(x(t), u(t))$ if the following holds for a finite $T\geq0$:
\begin{align}\label{eqn:defn_fixed_ensemble}\\[-2em]
    R( \varphi_T(u,{E}(x^i)))=y^i, \forall x^i \in \mathcal{X},
\end{align}
where the flow $\varphi_t(u,\cdot)$ is generated by the system at the control function $u$.  
\end{definition} 

We suppress the subscript $T$ in the notation at $t=T$ for simplicity. We call $R(\varphi(u,E(\cdot)))$ {\em the end-point mapping}. that is, we have the following end-point mapping for a given $u$ at $x^i$:
\begin{align}\\[-2em]
    x^i \in \mathbb{R}^{n} \xrightarrow{E} \Bar{x}^i \in \mathbb{R}^{\Bar{n}}  \xrightarrow{\varphi_T(u,\cdot)} \Bar{y}^i \in \mathbb{R}^{\bar{n}} \xrightarrow{R} y^i  \in \mathbb{R}^{n_o} \\[-2em]
\end{align}
In other words, the dynamical system has memorized the ensemble if the end-point mapping maps each $x^i \in \mathbb{R}^{n}$ to the corresponding $y^i \in \mathbb{R}^{n_o}$. The learning problem turns into a ``multi-motion planning'' problem (in which we look for a control function $u$ that steers all $\Bar{x}^i$ in $\mathcal{X}$ to the corresponding $\Bar{y}^i$). One can also see that, from the deep learning perspective, the memorization property is equivalent to the universal interpolation property of neural networks~\cite{tabuada2022universal}.

For convenience, we assume that $E$ is the identity function, meaning $\overline{n}=n$, but our result holds for any injective up-lift function. Thus, we interchangeably use $x^i$ and $E(x^i)$ for any $x^i \in \mathcal{X}$. If $n=n_o$ and $R$ is the identity map, we call the problem the control of an ensemble of points with {\em fixed end-points}. Otherwise, we call the problem the control of an ensemble of points with {\em partially constrained end-points}. In this work, we consider the latter. For the sake of simplicity, we let the map $R$ be the orthogonal projection of a point in $\mathbb{R}^{n}$ onto the subspace $\mathbb{R}^{n_o}$; precisely, $R: x \in \mathbb{R}^{n} \mapsto C x \in \mathbb{R}^{n_o}$ where $C = [  0_{n_o \times n-n_o}~I_{n_o \times n_o} ]\in\mathbb{R}^{n_o \times n}$.  We define per-sample cost functional $\mathcal{J}^i(u)$ for a given point $x^i$ as follows:
 \begin{align}\label{eqn:per_sample}\\[-2em]
     \mathcal{J}^i(u) = \frac{1}{2}\|C\varphi(u,{x}^i) - y^i \|^2 
 \end{align}\\[-1.5em]
Then, the cost-functional for the control of an ensemble of points with partially constrained end-points including regularization is defined as follows:
\begin{align}\label{eqn:cost_cont}\\[-2em]
    \mathcal{J}( u , \mathcal{X} ) := \sum_{i=1}^{q} \| C\varphi(u,x^i) - y^i \|^2   + \lambda \int_{0}^{T} \|u(\tau)\|^2 d\tau\\[-2em]
\end{align}
where $\lambda$ is some regularization coefficient. If we had a single initial point, the minimization of the functional~\eqref{eqn:cost_cont} subject to the dynamic~\eqref{eqn:control_system} would be the Bolza optimal control problem with a fixed time~\cite{liberzon2011calculus}. However, we have an ensemble of points $\mathcal{X}$, which might be a finite set of points or a continuum.

\paragraph{$q$-Folded Method:}

The authors in~\cite{agrachev2022control} discuss the following method, known as {\em the $q$-folded method}, to find the control function $u$. This turns the problem into a Bolza problem with the following steps: Copy all the points in $\mathcal{X}$ into an $q$-tuple, denoted by $X_0$. 
Similarly, stack all the output points into a vector of dimension $n_oq$, denoted by $Y$, and copy $n$-dimensional dynamics in~\eqref{eqn:control_system} $q$-times, creating an $nq$-dimensional vectors $F(u(t),X(t))$. Then, they find the flow ${\Vec{\varphi}}_t(u,\cdot):\mathbb{R}^{nq} \times \mathbb{R} \to \mathbb{R}^{nq}$, by solving the minimization problem associated with the following functional for a finite time $T$:
\begin{equation}\label{eqn:cost_ensemble}\\[-1em]
    \mathcal{J}(u, X_0 ):= \|\Lambda(C)\Vec{\varphi}_T(u,X_0) - Y \|^{2} + \int_{0}^{T} \|u(\tau) \|^2 d\tau  
\end{equation}
where $\Lambda(C)\!:=\!\operatorname{diag}(C,\!\cdots\!,C) \in \mathbb{R}^{n_0q \times nq}$ subject to $\dot{X}(t) = F(u(t),X(t))$, $X(0)=X_0$. The stacked structure of $\dot{X}(t)$ and the definition of $\Lambda(C)$ ensure that the flow $\vec{\varphi}_t(u,\cdot)$ is constructed by concatenating $q$ instances of the flow $\varphi_t(u,\cdot)$. 

One can easily see that $q$-folded method requires solving a boundary value problem on $qn$ dimension. Solving a boundary value problem by using the shooting method requires solving backward-forward initial value problem for CODE~\cite{keller1976numerical} at each iteration. An explicit solver of an initial value problem for CODE on $\mathbb{R}^n$ has the complexity of $\mathcal{O}(n^2 N)$ per iteration with $N$ time-step discretization. Consequently, the $q$-folded method has complexity of $\mathcal{O}( n^2 q^2 N )$ per iteration. Therefore, as the size of the training set, $q$, grows, the complexity of the $q$-folded method increases quadratically.

\section{Main Results}
In this section, we introduce the main result of this paper.
\vspace{-4mm}
\paragraph{Controllability on Partially Constrained Ensembles:} We first state a simple condition for the existence of a control function $u$ that memorizes a given ensemble $(\mathcal{X},\mathcal{Y})$. 
To be more precise, we introduce the following subspaces of our control space:
\begin{align}\\[-2.5em]
        U(x^i,y^i) := \{ u \in L_{\infty}([0,T], \mathbb{R}^{n\times n} )| R(\varphi(u,x^i))=y^i\}\\[-2em]
\end{align}
Then, we need to show that there exists a control function $u \in \bigcap_{i=i}^q U(x^i,y^i)$ to prove the model has memorization property. 

Let $g_1$ and $g_2$ be differentiable vector fields in $\mathcal{M} \in \mathbb{R}^{n}$. We call the Lie bracket of $g_1$ and $g_2$~\cite{brockett2014early}, denoted by $[g_1,g_2](x)$ , the vector field:
\begin{align}\\[-1.5em]
    [g_1,g_2](x) := \frac{\partial g_2(x)}{\partial x} g_1(x) - \frac{\partial g_1(x)}{\partial x} g_2(x) , \mbox{for } x\in \mathcal{M} \\[-1.5em]
\end{align}




\begin{definition}[Control distributions]\label{defn:control_dist} Let $\dot{x}=f(x,u), u \in L_{\infty}([0,T],\mathbb{R}^{n\times n})$, we associate to it the sets of vector fields defined recursively
\begin{align}\\[-2em]
    \mathcal{F}^k=\mathcal{F}^{k-1} \cup\left\{  [g_i,g_j](x)  \mid g_i, g_j \in \mathcal{F}^{k-1} \right\} \\[-2em]
\end{align}
for $k\geq1$ with the set of control vectors $\mathcal{F}^0=\{ f(x,u) | u \in  L_{\infty}([0,T], \mathbb{R}^{n\times n}) \}$. The corresponding distributions at point $x \in \mathbb{R}^n$ are then
\begin{align}\\[-2em]
    \mathcal{D}_x^k(\mathcal{F})=\operatorname{span}\left\{g(x) \mid g \in \mathcal{F}^k\right\}\\[-2em]
\end{align}
\end{definition}

 
 One can see that the distributions are nested $\mathcal{D}_x^k(\mathcal{F}) \subseteq \mathcal{D}_x^{k+1}(\mathcal{F})$. The control distribution $\mathcal{D}_x^k(\mathcal{F})$ is said to be involutive  if $\mathcal{D}_x^k(\mathcal{F}) = \mathcal{D}_x^{k+1}(\mathcal{F})$~\cite{brockett2014early}. We call $\mathcal{F}$ {\em bracket-generating} if $\mathcal{D}_x^\infty(\mathcal{F})$ spans $T_x E(\mathcal{M})$ for all $x\!\in\!E(\mathcal{M})$.  

We denote $q$ times Cartesian product of $E(\mathcal{M})$ by $E(\mathcal{M})^q:= E(\mathcal{M}) \times \cdots \times E(\mathcal{M})$. A finite ensemble of points is a $q$-tuple $X=\begin{bmatrix} E(x^1)^\top , \cdots , E(x^q)^\top \end{bmatrix}^\top \in E(\mathcal{M})^q \subseteq \mathbb{R}^{nq} $. We define the set $\Delta^q :=\{ [ E(x^1)^\top , \cdots , E(x^q)^\top ]^\top \in E(\mathcal{M})^q \rvert $ $E(x^i) = E(x^j) \mbox{ for } i \neq j \} $. Let $E(\mathcal{M})^{(q)}:= E(\mathcal{M})^{q} \setminus \Delta^q$ be the complement of   $\Delta^q$ on $E(\mathcal{M})^q$. Please note that if $\operatorname{dim} E(\mathcal{M}) > 1$, then $E(\mathcal{M})^{(q)}$ is an open connected subset and submanifold of $E(\mathcal{M})^{q}$~\cite{agrachev2022control}.


The set of control vectors of the $q$-folded system, denoted by $\tilde{\mathcal{F}}^0$, is then
\begin{align} \\[-1.75em]
   \tilde{\mathcal{F}}^0=\{ [f(u,x^1)^\top,f(u,x^2)^\top,\cdots,f(u,x^q)^\top]^\top \in \mathbb{R}^{nq} | u \in  L_{\infty}([0,T], \mathbb{R}^{n\times n})\}.  \\[-1.75em]
\end{align}
We note that the control function $u$ still belongs to $L_{\infty}([0,T], \mathbb{R}^{n\times n})$, and not $L_{\infty}([0,T], \mathbb{R}^{nq\times nq})$. Therefore, having $\mathcal{D}_X^\infty(\tilde{\mathcal{F}})=T_{X} E(\mathcal{M})^{(q)} (= \otimes_{i=1}^q T_{x^i} E(\mathcal{M}))$ for all $X \in E(\mathcal{M})^{(q)}$ is stronger requirement than having $\mathcal{D}_x^\infty({\mathcal{F}})=T_{x} E(\mathcal{M})$ for all $x \in E(\mathcal{M})$.

\begin{lemma}\label{lem:ensemble_cont} Assume that the ensemble $\mathcal{X}$ consists of finite pairwise distinct points and $n > n_o$. For the readout map $R(x)=Cx$, if the set of control vector fields of $q$-folded system is bracket-generating in $E(\mathcal{M})^{(q)}(= E(\mathcal{M})^q \setminus \Delta^q)$, then there exists a control function $u$ and a finite time $T\geq0$ such that the system~\eqref{eqn:control_system} memorizes the ensemble $(\mathcal{X},\mathcal{Y})$ by the control function $u$. 
\end{lemma}

\begin{proof}
First consider that  $i \neq j \Leftrightarrow y^i\neq y^j$. In this case,  the result trivially follows from~\cite[Prop 6.1]{agrachev2020control}. 
Now assume that it is not the case; then, since the cardinality of $\mathcal Y$ is finite, we can always construct a set $\mathcal{Y}'\subset \mathbb{R}^{n}$, with elements $\tilde{y}^i$ having the property that  $i \neq j \Leftrightarrow \tilde{y}^i\neq \tilde{y}^j$ and $R(\tilde{y}^i)=y^i$ for all $i$. We can then appeal to~\cite[Prop 6.1]{agrachev2020control} to conclude.\qed\end{proof}

\vspace{-4mm}
\paragraph{Tuning without forgetting:} We describe here the core of our approach to developing numerical methods. 
Let $\mathcal{X}^j=\{ x^i \in \mathcal{X} | i = 1,2,\cdots,j\}$ be a subset of the ensemble $\mathcal{X}$ (called {\em batch} or {\em sub-ensemble}). Let $\mathcal{Y}^j$ be the corresponding batch of labels. We denote by a superscript the value of the control at a given iteration, i.e.,  $u^k$ is the control function at the $k$th iteration of the algorithm. Assume that $u^k$ has memorized the ensemble $(\mathcal{X}^j,\mathcal{Y}^j)$ for the model~\eqref{eqn:control_system}. Expand the ensemble by adding the point $x^{j+1}$ with its corresponding label $y^{j+1}$. Clearly, it does not necessarily hold that $C\varphi(u^k,x^{j+1})=y^{j+1}$. We propose an iterative method to find a control $u^*$ such that $C\varphi(u^*,x^{i})=y^{i}$ for all $x^i \in \mathcal{X}^{j+1}(=\mathcal{X}^{j} \cup \{x^{j+1}\})$.

\begin{definition}\label{defn:lwf} Consider an ensemble $(\mathcal{X},\mathcal{Y})$. Assume that the control $u^k$  has memorized the sub-ensemble $(\mathcal{X}^j,\mathcal{Y}^j)$ for~\eqref{eqn:control_system} for some $j<|\mathcal{I}|$. If the update $\delta u^k$ satisfies the following: 
\begin{enumerate}
\item\label{itm:cond_1} $ \mathcal{J}^{j+1}(u^{k}+\delta u^k) \leq \mathcal{J}^{j+1}(u^k) $
\item\label{itm:cond_2} $R\left( \varphi(u^{k}+\delta u^k,x^i) \right) =y^i + o(\delta u^k), \forall x^i \in \mathcal{X}^j$ 
\end{enumerate}
then the control function $u^{k+1}(:=u^k+\delta u^k)$  has been {\em tuned} for $\mathcal{X}^{j+1}$ {\em  without forgetting} $\mathcal{X}^j$.
\end{definition}

Paraphrasing, the definition says that we need to select an update $\delta u^{k}$ satisfying the following two conditions: $(i)$ it decreases the per-sample cost functional for the additional point $x^{j+1}$, equivalently, the control system with the updated control $u^{k+1}$ steers the point $x^{j+1}$ to a point whose projection onto the output subspace $\mathbb{R}^{n_o}$ gets closer to the label $y^{j+1}$, and $(ii)$ the points in $\mathcal{X}^j$ are mapped to points whose projection onto the output subspace is within $o(\delta u^k) $ of their corresponding labels. To be more precise, in {\em tuning without forgetting}, we aim to minimize the per-sample cost for the new point $x^{j+1}$, denoted by $\mathcal{J}^{j+1}(u)$, with $u \in \bigcap_{i=1}^j U(x^i,y^i)$.

\paragraph{A projected gradient descent method:} Consider the flow $\varphi_t(u,\cdot)$. It yields the trajectory $t \mapsto \varphi_t(u,x^i)$ of~\eqref{eqn:control_system} with control $u$ and initialized at $x^i$ at $t=0$. The first-order variation of the trajectory $\varphi_t(u,x^i)$ in $\delta u$ is defined as $\delta \varphi_t(u,x^i):=\varphi_t(u+\delta u,x^i)-\varphi_t(u,x^i)$. Under Condition~\ref{itm:cond_2} in Definition~\ref{defn:lwf}, it should hold that $C\delta\varphi_T(u,x^i)=0$ for all $x^i\in \mathcal{X}^j$ up to  first order in $\delta u$. 

It is well known that  $\delta \varphi_t(u,x^i)$ obeys the linear time-varying equation, which is simply the linearization of the control system~\eqref{eqn:control_system} about the trajectory $\varphi_t(u,x^i)$. Thus, we define the following property:
\begin{definition}[Linearized Controllability Property]
 We say the system $\dot x(t)=f(x(t),u(t))$ has the Linearized Controllability Property (LPC) at $x^i$ for all $u \in L_{\infty}([0,T], \mathbb{R}^{n \times n} )$ if the linear time varying system:      
 \begin{align}\label{eqn:defn_ltv}
    \dot{z}(t) =  \left( \frac{\partial f(x,u)}{\partial x} \rvert_{(x=\varphi_{t}(u,x^i),u)} \right)z(t) + \left( \frac{\partial f(x,u)}{\partial u}  \rvert_{(x=\varphi_{t}(u,x^i),u)} \right) v(t),
\end{align}
$ v(t) \in L_{\infty}([0,T], \mathbb{R}^{n \times n} )$ is controllable. 
\end{definition}


Denote by  $\Phi_{(u,x^i)}(t,\tau)$ the state transition matrix of~\eqref{eqn:defn_ltv}.

\begin{lemma}\label{lem:variation_t}
Suppose that a given control function $u$ has memorized the pair of points $(x^i,y^i)$ for the model~\eqref{eqn:control_system}. Then
\begin{equation*}
     \delta \varphi_t(u,x^i) = \int_{0}^{t}  \Phi_{(u,x^i)}(t,\tau)  \frac{\partial f( x, u ) }{\partial u} \rvert_{(x=\varphi_{\tau}(u,x^i),u)}  \delta u(\tau)  d\tau.
\end{equation*}
up to first order in $\delta u(t)$.
\end{lemma}

\begin{proof}
For a small variation of the control function $u(\tau)$, denoted by $\delta u(\tau)$, we have the following:
\begin{align}\label{eqn:new_flow_dist}\\[-2em]
    \dot{x}(t) + \delta \dot{x}(t) =   f \left(x(t) + \delta x(t) ,  u(t) + \delta u(t)  \right)  \\[-2em]
\end{align}
Taking the {\em first} order Taylor expansion of~\eqref{eqn:new_flow_dist} around the trajectory $x(t)=\varphi_t(u,x^i)$ and  subtracting $\varphi_t(u,x^i)$, we get
\begin{align}\label{eqn:only_delta_x}
   \frac{d~\delta \varphi_t(u,x^i)}{dt}= \frac{\partial f(x,u)}{\partial x}\delta \varphi_t(u,x^i) + \frac{\partial f(x,u)}{\partial u}  \delta u
\end{align}
up to first order in $\delta u$. To emphasize the linearity of the system above in the control update $\delta u$, we can introduce the notation  $z(t)=\delta \varphi_t(u,x^i)$.Then, from the introduced notation, one can easily see that~\eqref{eqn:only_delta_x} matches with~\eqref{eqn:defn_ltv}.
Using the variation of constants formula~\cite{liberzon2011calculus}, we have the following:
\begin{align*}\label{eqn:const_variation}
    z(t)= \Phi_{(u,x^i)}(t,0)z(0)+\int_{0}^{t}\Phi_{(u,x^i)}(t,\tau)  \frac{\partial f(x,u)}{\partial u} \rvert_{(x=\varphi_{\tau}(u,x^i),u)} \delta u(\tau)~d\tau
\end{align*}
We have $z(0)=0$ since we have $\varphi_0(u+\delta u,x^i)=x^i$ and $\varphi_0(u ,x^i)=x^i$. This completes the proof.\qed\end{proof}

Based on Lemma~\ref{lem:variation_t}, we define an affine operator from the space of bounded functions over the time interval $[0,T]$, $\delta u(t) \in L_{\infty}([0,T], \mathbb{R}^{n \times n})$, to $\mathbb{R}^{n_o}$, mapping a control variation to the resulting variation in the end-point of the trajectory see through the readout map $R$. This operator is defined as
\begin{equation}\label{eqn:operator}
    \mathcal{L}_{(u,x^i)}(\delta u):= R \left( \int_{0}^{T} \Phi_{(u,x^i)}(T,\tau)  \frac{\partial f(x(\tau),u(\tau))}{\partial u} \delta u(\tau)d\tau \right)  
\end{equation}
for $x(\tau)=\varphi_{\tau}(u,x^i)$. Next, we let 
$$\mathcal{K}(u,x^i):=\operatorname{span}\{\delta u\in L_{\infty}([0,T],\mathbb{R}^{n \times n}) \mid  \mathcal{L}_{(u,x^i)}(\delta u) = 0 \}$$ 
be the kernel of the operator $\mathcal{L}_{(u,x^i)}(\cdot)$. Then, we define the intersection of the kernel $\mathcal{K}(u,x^i)$ for all $x^i\in\mathcal{X}^j$ as follows:
$$
\mathcal{K}(u,\mathcal{X}^j):=\operatorname{span}\{ \delta u \in L_{\infty}([0,T],\mathbb{R}^{n \times n}) \mid \delta u \in\bigcap_{ x^i\in \mathcal{X}^j}\mathcal{K}{(u,x^i)} \}
$$
We define the gradient of the given per-sample cost functional for the sample $x^{i}$ at control $u$, as the first order variation of $\mathcal{J}^{i}(u)$ in $\delta u$, precisely, $ \nabla_{u} \mathcal{J}^{i}(u):= \mathcal{J}^{i}(u+\delta u)-\mathcal{J}^{i}(u)$. Then, one can see that we have the following: 
\begin{equation}
    \nabla_{u(t)} \mathcal{J}^{i}(u) := \delta  \varphi_t^\top(u,x^{i}) C^\top \left( C \varphi(u,x^{i}) - y^{i} \right)
\end{equation}
Note that it is a function of time $t$ via the first-order variation $\delta \varphi_t(u,x^{i})$. We define the projection of $\nabla_{u(t)} \mathcal{J}^{j+1}(u)$ on a given subspace of functions $\mathcal{K}(u,\mathcal{X}^j)$ as the solution of the following optimization problem:
\begin{equation}
   \operatorname{proj}_{\mathcal{K}(u,\mathcal{X}^j)}\nabla_{u(t)} \mathcal{J}^{j+1}(u):={\arg\min}_{d(t) \in \mathcal{K}(u,\mathcal{X}^j)} \int_0^T  \| d(\tau) -  \nabla_{u(\tau)} \mathcal{J}^{j+1}(u) \|^2 d\tau 
\end{equation}

Now, we can state the main result:

\begin{theorem}\label{main:approximation}
Consider  model~\eqref{eqn:control_system} and suppose that the control vector fields of $q$-folded  system~\eqref{eqn:control_system} are bracket-generating in $E(\mathcal{M})^{(q)}$,  and the control function $u^k$ has memorized the ensemble $(\mathcal{X}^j,\mathcal{Y}^j)$. Assume the space of controls that memorize $(\mathcal{X},\mathcal{Y})$ is connected Banach submanifold of $L_{\infty}([0,T],\mathbb{R}^{n \times n})$. If $\delta u^k$ is selected as $\operatorname{proj}_{\mathcal{K}(u,\mathcal{X}^j)}\nabla_{u(t)} \mathcal{J}^{j+1}(u)$, then the control function $u^{k+1}(:=u^k+\delta u^k)$ for $\mathcal{X}^{j+1}$ has been {\em tuned without forgetting} $\mathcal{X}^j$ up to the first order. 
\end{theorem} 

We state a result that shows that gradient descent to minimize $\mathcal{J}^{j+1}(u)$ subject to $u \in \bigcap_{i=1}^j U(x^i,y^i)$ is a well-founded approach:

\begin{theorem}\label{thm:banach_manifold}
Assume the system $\dot x(t) = f(x(t), u(t))$ on a manifold $E(\mathcal{M})$ has the LPC for each $x^i$, and that the set of control vector fields of $q$-folded system is bracket-generating in $E(\mathcal{M})^{(q)}$. Assume that $E(\mathcal{M})$ is connected and that the fundamental group $\pi_1(E(\mathcal{M}))=0$. Then, the space of controls that memorize $(\mathcal{X}, \mathcal{Y})$ is connected Banach submanifold of $L_{\infty}([0,T],\mathbb{R}^{n \times n})$ of  finite-codimension. 
\end{theorem}

Due to our focus on algorithms and because of space constraints, we omit the proof of Theorem~\ref{thm:banach_manifold} here. We will provide it in future work. The reader can take its conclusion as an assumption of Theorem~\ref{main:approximation}. 

Now, we can state the proof of result used in our algorithm.

\begin{proof}[Theorem~\ref{main:approximation}] From Lemma~\ref{lem:ensemble_cont}, we know that there exists a control $u^{k+1}$ such that $$C\varphi(u^{k+1},\mathcal{X}^{j+1})=\mathcal{Y}^{j+1}.$$
From definition of $\mathcal{L}_{(u^k,x^i)}(\cdot)$ and Lemma~\ref{lem:variation_t}, we have:
\begin{align*}
    C\varphi(u^{k}+\delta u^k,x^i) = C\varphi(u^{k},x^i)  +  \mathcal{L}_{(u^k,x^i)}(\delta u^k) ,  \forall x^i \in \mathcal{X}^j
\end{align*}
up to first order in $\delta u^k$. Then, we project $\nabla_u\mathcal{J}^{j+1}(u^k)$ onto $\mathcal{K}(u^k,\mathcal{X}^j)$, that is, we select $\delta u^k$ as $\delta u^k=\operatorname{proj}_{\mathcal{K}(u^k,\mathcal{X}^j)}\nabla_u\mathcal{J}^{j+1}(u^k)$. The selection satisfies Condition~\ref{itm:cond_1} in Definition.~\ref{defn:lwf}. 
By the construction of $\mathcal{K}(u^k,\mathcal{X}^j)$  and from the memorization assumption, we have that $\mathcal{L}_{(u^k,x^i)}(\delta u^k)=0,\forall x^i \in \mathcal{X}^j$ and $C\varphi(u^{k},x^i)=y^i,\forall x^i \in \mathcal{X}^j$, respectively. Then, we have:
\begin{align}\label{eqn:cond_2_rewrite}
   C\left( \varphi(u^{k}+\delta u^k,x^i) \right) =   y^i  +   o( \delta u^k ) , \forall x^i \in \mathcal{X}^j  
\end{align}
It matches the Condition~\ref{itm:cond_2} in Definition~\ref{defn:lwf}, and this completes the proof.\qed
\end{proof}
We emphasize that we {\em only} consider the intersection of the kernel $\mathcal{K}(u^k,x^i)$ for all the points $x^i \in \mathcal{X}^j$ because we need to keep the first-order variation of the end-point mapping $C(\varphi(u^{k},x^i))$ zero for all points $x^i$ in $\mathcal{X}^j$ while we allow variation on the end-point mapping for the rest of the points. 

\section{Numerical Method for Tuning without Forgetting}
Building on the results of the previous section, we describe in this section a numerical algorithm to {\em tune} a control function {\em without forgetting} $\mathcal{X}^j$ up to the first order. The algorithm comprises three main phases, on which we elaborate below. Assume that a control $u^0$ has memorized  $(\mathcal{X}^j,\mathcal{Y}^j)$ for a given control affine system. Consider an expansion of the ensemble, $\{x^{i}\}_{i=j+1}^q$. For the iteration $k=1$, we let $u^1=u^0$. We assume that $T=1$ without loss of generality. For the sake of exposition, we also rely on an explicit Euler method to solve the ODEs involved; using different solvers would entail a trivial modification. We let $N$ be the number of discretization points.

For the sake of notation simplicity, we assume that the control system is affine in the control: $
     \Dot{x}(t)=\sum_{d=1}^{p} u_d(t) f_d(x(t)) , u_d(t)\in \mathbb{R}, u \in L_{\infty}([0,T], \mathbb{R}^p)$ for some smooth functions $f_d(\cdot) \in \operatorname{Vect}(E(\mathcal{M})), d=1,\cdots,p$. However, one can easily apply the algorithm by vectorizing the control function in the matrix form~\eqref{eqn:control_system}.

We let  $u^k \in \mathbb{R}^{pN}$ be the discretization of a time-varying control function and denote the corresponding discretized trajectory by $\varphi_{[1:N]}(u^k,x^i)\in\mathbb{R}^{n \times N}$. We let ${\delta u^k} \in \mathbb{R}^{pN}$ be the discretization of a time-varying control variation.

\paragraph{Approximation of $\mathcal{L}_{(u,x^i)}(\cdot)$:} We first provide a method to compute a numerical approximation of $\mathcal{L}_{(u^1,x^i)}(\cdot)$ for all $x^i \in \mathcal{X}^j$. This is Algorithm~\ref{alg:state_transisition} below, which computes the numerical approximation of $d_u \varphi(u^k,x^i)$. 
\vspace{-3mm}
\begin{algorithm}[htp]
\caption{Approximation of $\mathcal{L}_{(u,x^i)}(\cdot)$}\label{alg:state_transisition}
\begin{algorithmic}[1]
\STATE { $\mbox{\textbf{Input: }} u, x^i$}
\STATE {$z \gets {\varphi}_{[1:N]}(u,x^i) , \Phi_{N+1} \gets I_{n \times n}  $}    
    \FOR{$\ell =N$ to $1$}    
      \STATE $F[\ell] \gets f(u[\ell] , z[\ell])  $
        \STATE $\Phi_\ell \gets \Phi_{\ell+1} ( I + \frac{ \partial  F[\ell] }{ \partial x} ) $

    \ENDFOR
    \RETURN $C [\Phi_1 F[1], \Phi_2 F[2] , \cdots, \Phi_N F[N] ] $      

\end{algorithmic}
\end{algorithm}
\vspace{-3mm}
The algorithm iteratively computes the state transition matrix $\Phi_{(u^k,x^i)}(T,\tau)$ of the system in~\eqref{eqn:defn_ltv} for a given initial point $x^i$ and a control function $u^k$. Then, $\Phi_{(u^k,x^i)}(T,\tau)$ is multiplied with the matrix of vector fields in~\eqref{eqn:control_system} and the product is stored in a matrix, which is the numerical approximation of $d_u \varphi(u^k,x^i)$. Let $L_i\in \mathbb{R}^{n_o \times pN}$ be the output of Algorithm~\ref{alg:state_transisition} for a given $u^k$ and an initial point $x^i$. We  have $\mathcal{L}_{(u^k,x^i)}(\delta u^k) \approx L_i {\delta u^k}$. Then, the right kernel of the matrix $L_i$ is the numerical approximation of $\mathcal{K}{(u^k,x^i)}$.

\paragraph{Phase I:} We now implement the statement of Theorem~\ref{main:approximation}. We state the Phase I in Algorithm~\ref{alg:kernel_gradient}. We need to construct $\mathcal{K}{(u^k,\mathcal{X}^j)}$ for a given $j$. We first compute $L_i$ for all $x^i \in \mathcal{X}^j$ by using Algorithm~\ref{alg:state_transisition}. Then, we column-wise concatenate them and place the concatenated matrix into the first block of $L\in\mathbb{R}^{n_oq \times pN}$ and fill the remaining entries with zeros. One can easily see that the right kernel of the matrix $L$, denoted by $\mathcal{N}(L)$,  is a numerical approximation of $\mathcal{K}{(u^1,\mathcal{X}^j)}$ for a given $j$. Following Theorem~\ref{main:approximation}, we need to project $\nabla_u \mathcal{J}^{j+1}(u^k) \in \mathbb{R}^{pN}$ onto $\mathcal{N}(L)$, which is a linear operation. We have the following for the update control $u^{k+1}$:
\begin{align}\label{eqn:kernel_project}
    u^{k+1} =  u^k - \alpha^k \operatorname{proj}_{\mathcal{N}(L)} \nabla_u \mathcal{J}^{j+1}(u^k) \\[-2em]
\end{align}
where $\alpha^k \in \mathbb{R}^{+}$ is a step size. We repeat the iteration in~\eqref{eqn:kernel_project}, and we compute the matrix $L_{i}$ for $i=1,\cdots,j$ for the control $u^k$ to replace corresponding blocks in the matrix $L$ until the per-sample loss converges. Once it converges, we pick the next point $x^{j+2}$ and compute the matrix $L_{i}$ for $i=1,\cdots,j+1$ to replace corresponding blocks in the matrix $L$. We then continue picking new points.
            
\vspace{-3mm}
\begin{algorithm}
\caption{Phase I: Kernel Projected Gradient Descent}\label{alg:kernel_gradient}
\begin{algorithmic}[1]
\STATE {\strut$L_i \gets \mathrm{Algorithm1}(u^0 , x^i),\forall i\in\mathcal{I}^j$}
\STATE {$~u \gets u^0$}

\STATE $L \gets [L_1 ; L_2;\cdots ; L_j ; 0 ; \cdots ; 0 ]$ 
\FOR{$i=j+1$ to $q$}  
            \REPEAT            
            \STATE $u \gets u - \alpha \operatorname{proj}_{\mathcal{N}(L)} \nabla_u \mathcal{J}^{i}(u)$ 
            \STATE $L_\ell \gets \mathrm{Algorithm1}(u,x^\ell), \mbox{for } \ell=1,\cdots,i-1$
            \STATE $\mbox{Update the corresponding block } L_\ell \mbox{ in } L,\forall \ell \in \mathcal{I}^{i-1}$  
            \UNTIL{\textbf{convergence}}
\ENDFOR
\RETURN $L,u^k$
\end{algorithmic}
\end{algorithm}
\vspace{-3mm}
\begin{remark}\label{rem:scalability}
If we consider $\mathcal{X}^0$ as an empty set and $u^0=\Vec{0}$, it results in $L$ to be a zero matrix. Thus, one can use the proposed algorithm to design a control function $u$ that sends given initial states to desired final ones. This is an alternative to the $q$-folded method.  
\end{remark}

In the next two phases, we focus on minimizing the cost-functional $\mathcal{J}(u,\mathcal{X})$ in~\eqref{eqn:cost_cont}, which comprises two sub-cost functionals: the sum of per-sample costs and the $L^2$ norm of the control $u$. First, we minimize the $L^2$ norm of the control $u$, and then we minimize the sum of per-sample costs. We call the consecutive execution of Phase II and III {\em refinement rounds}. One can employ multiple refinement rounds to reduce the cost. 
\vspace{-1mm}
\paragraph{Phase II} In this phase, we minimize the $L^2$ norm of the control $u$. We project the gradient of the $L^2$ norm of the control function onto the subspace of functions $\mathcal{K}(u^k,\mathcal{X})$ at each iteration. We state Phase II in Algorithm~\ref{alg:kernel_gradient_u_norm}.
\paragraph{Phase III} In this phase, we aim to refine the control $u$ to steer all the points closer to their associated end-points. Let $\mathcal{P}$ be the number of iterations per sample in the algorithm. For a given point $x^i$, we first call Algorithm~\ref{alg:state_transisition} to compute $L_\ell$ for all $\ell \in \mathcal{I} \setminus \{i\}$. Notably, we do not update the kernel for $x^i$. Then, we update the matrix $L$ accordingly and compute the projected gradient. Then, we pick the next point $x^{i+1}$. We repeat these steps until the last point in the training set. Then, we again pick $x^1$ and repeat the steps above $\mathcal{P}$ times.
\vspace{-4mm}
\begin{algorithm}
\caption{Phase II: Regularization}\label{alg:kernel_gradient_u_norm}
\begin{algorithmic}[1]
\STATE {\strut $L,u \gets \mathrm{Algorithm2}()$}
    \FOR{$i=1$ to $q$}
        \REPEAT
        \STATE $L_\ell \gets \mathrm{Algorithm1}(u,x^\ell), \forall x^\ell \in \mathcal{X} $
        \STATE $\mbox{Update the corresponding block } L_\ell \mbox{ in } L,\forall \ell\in\mathcal{I}$
        \STATE $u \gets u -\alpha \operatorname{proj}_{\mathcal{N}(L)}(u)$
        \UNTIL{\textbf{convergence}}
    \ENDFOR    
\RETURN $L, u$
\end{algorithmic}
\end{algorithm}
\vspace{-11mm}
\begin{algorithm}
\caption{Phase III: Refinement }\label{alg:kernel_gradient_3}
\begin{algorithmic}[1]
\STATE {\strut  $L,u \gets \mathrm{Algorithm3}()$}
\FOR{$m=1$ to $\mathcal{P}$ }
\FOR{$i=1$ to $q$}
        \STATE $L_\ell \gets \mathrm{Algorithm1}( u,  x^\ell ) , \forall x^\ell \in \mathcal{X} \setminus \{x^i\}  $
        \STATE $\mbox{Update the corresponding block } L_\ell \mbox{ in } L,\forall \ell \in\mathcal{I}\setminus\{i\}$
        \STATE $u \gets u - \alpha \operatorname{proj}_{\mathcal{N}(L)} \nabla_u \mathcal{J}^{i}(u)$ 

\ENDFOR
\ENDFOR
\RETURN $u$
\end{algorithmic}
\end{algorithm}
\vspace{-11mm}
\section{A Computational Example}

In this section, we provide computational examples. We consider the model
\begin{equation}\label{eqn:examp_system}
    \dot{x}(t) = W_2(t)\tanh{\left ( W_1(t)x(t) + b_1(t) \right) } + b_2(t)   
\end{equation}
where $W_1(t),W_2(t) \in L_{\infty}([0,T],\mathbb{R}^{\bar{n} \times \bar{n}})$ and $b_1(t),b_2(t) \in L_{\infty}([0,T],\mathbb{R}^{\bar{n} })$. We discretize trajectories with $N=10$. For a given $x^i \in \mathcal{X} \in \mathbb{R}^2$, we consider the following end-point mapping. 
\begin{align}\\[-2.5em]
    x^i \in \mathbb{R}^{2} \xrightarrow{E} [ {x^i}^\top, 0_{1\times 6}]^\top \in \mathbb{R}^{8}  \xrightarrow{\varphi_T(u,\cdot)} \Bar{y}^i \in \mathbb{R}^{8} \xrightarrow{C=[0,\cdots,0,1]} y^i=   \begin{cases} 
      -1 &, \|x^i\|_2^2 \leq 1 \\
      1 &,   \|x^i\|_2^2 > 1
   \end{cases}\\[-3em]
\end{align}

In words, we consider a unit radius ball in $\mathbb{R}^2$ centered at the origin. If a point is inside the ball, we aim to steer it to the hyperplane $x_8=1$, otherwise, to the hyperplane $x_8=-1$ (a similar example is provided in~\cite{dupont2019augmented}). One can see that the set of control vector fields for the (overparametrized) $q$-folded control system~\eqref{eqn:examp_system} is bracket-generating in $E(\mathcal{M})^{(q)}$. We define the average error for a control $u$ on the ensemble $\mathcal{X}$ as 
\begin{align} \\[-2em]
   \mathcal{E}(u,\mathcal{X}):=\frac{1}{|\mathcal{X}|}\sum_{x^i\in\mathcal{X}}\| C\varphi(u,x^i)-y^i\|. \\[-2.6em]
\end{align}
We have a set $\mathcal{X}$ of cardinality $64$ whose elements are indexed by $\mathcal{I}$. First, we have that an initial training set (sub-ensemble) $\mathcal{X}^j$ is given for some $j$. We apply the $q$-folded method in Section~\ref{sec:prelim} to learn these samples and denote the control function that $q$-method gives by $u^0$. Then, we expand the training set $\mathcal{X}^j$ to $\mathcal{X}^i$ by adding samples in the order provided by $\mathcal{I}$. For convenience, we denote the difference of the subensembles $\mathcal{X}^i_j:=\{x^\ell \in \mathcal{X} : j < \ell \leq i \}(=\mathcal{X}^{i}\setminus \mathcal{X}^{j})$.

Due to expansion in the ensemble, we need to tune the control function $u^0$ to learn the new samples without forgetting the previous samples. Thus, we apply the {\em tuning without forgetting $\mathcal{X}^j$} method with multiple refinement rounds, which returns the control function $u^*$. We also compare our results with an existing method in the literature for fine-tuning, namely the {\em Penalty Method}, whose approach is to encourage the tuned control function, denoted by $\tilde{u}$, to remain close to the learned control function $u^0$ according to some metric~\cite{gouk2021distance} by updating the cost functional with a penalty term as follows:
\begin{align*}\label{eqn:local_learning} \\[-2em]
    \mathcal{J}( \Tilde{u} , \mathcal{X} ) := \sum_{i=1}^{M} \| C\varphi(\Tilde{u},x^i) - y^i \|^2   + \lambda \int_{0}^{T} \|\Tilde{u}(\tau)-u^0(\tau)\|^2 d\tau \\[-2em]
\end{align*}
Figure~\ref{fig:results} depicts the average error on different sets as a function of number of rounds for both algorithms. First, we observe that, in Fig~\ref{fig:results}, the average error on the new points $\mathcal{E}(u^0 ,\mathcal{X}^{i}_{j})$ is remarkably higher compared to $\mathcal{E}(u^0,\mathcal{X}^j)$ for any given training set at round $0$ (where $u^*=u^0$ and $\Tilde{u}=u^0$). As expected, it shows that the control $u^0$ has {\em not} learned the additional samples $\mathcal{X}^i_j$ for the given model.

The performance of methods designed to address catastrophic forgetting is measured by two metrics: {\em learning plasticity}, which measures the model's capacity to learn new tasks and {\em memory stability}, defined as the ability of a model to retain knowledge of old tasks~\cite{wang2024comprehensive}. 

For the learning plasticity, we measure the difference in the performance of model on a task between its joint training performance $\mathcal{E}(u^0,\mathcal{X}^i)$ and fine-tuning performance  $\mathcal{E}(\cdot,\mathcal{X}^i_j)$. We observe that both $\mathcal{E}(u^*,\mathcal{X}^i_j)$ and $\mathcal{E}(\tilde{u},\mathcal{X}^i_j)$ for $(i,j)=(64,52),(32,8),(32,25)$ are close to $\mathcal{E}(u^0, \mathcal{X}^i )$. It shows that the control $u^*$ and $\tilde{u}$ has learned the additional samples satisfactorily. For the memory stability, we compare the error on the previously learned points $\mathcal{E}(\cdot,\mathcal{X}^j)$ as the tuning continues. We observe that $\mathcal{E}(u^*,{\mathcal{X}}^j)$ for $j=8,16,25$ is close to $\mathcal{E}(u^0, \mathcal{X}^j)$. It shows that the control $u^*$ keeps the performance on the previously learned points $\mathcal{X}^j$ nearly constants. Comparing our method to the Penalty Method, we observe that $\mathcal{E}(\tilde{u}, \mathcal{X}^j)$ is higher than $\mathcal{E}(u^0, \mathcal{X}^j)$ for any given $j$. It shows that the control function $\Tilde{u}$ has learned the additional points but has forgotten the previously learned points. Thus, the penalty method does not directly address the issue of catastrophic forgetting.
\begin{figure}[t]
\centering
\subfigure[]
{\includegraphics[width=0.49\linewidth]{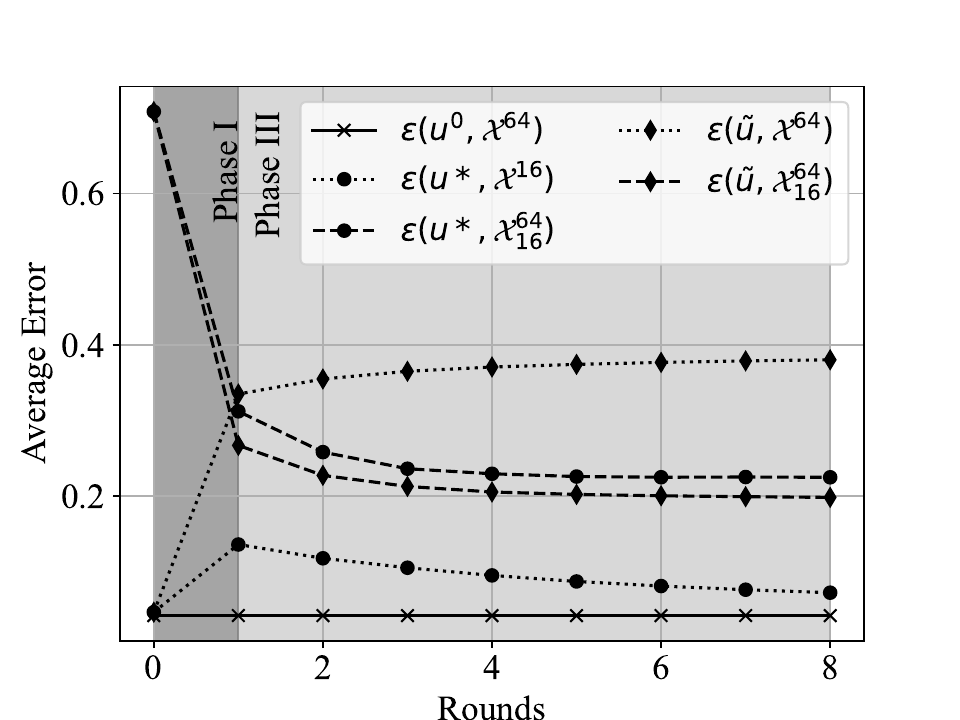}}
\subfigure[]
{\includegraphics[width=0.49\linewidth]{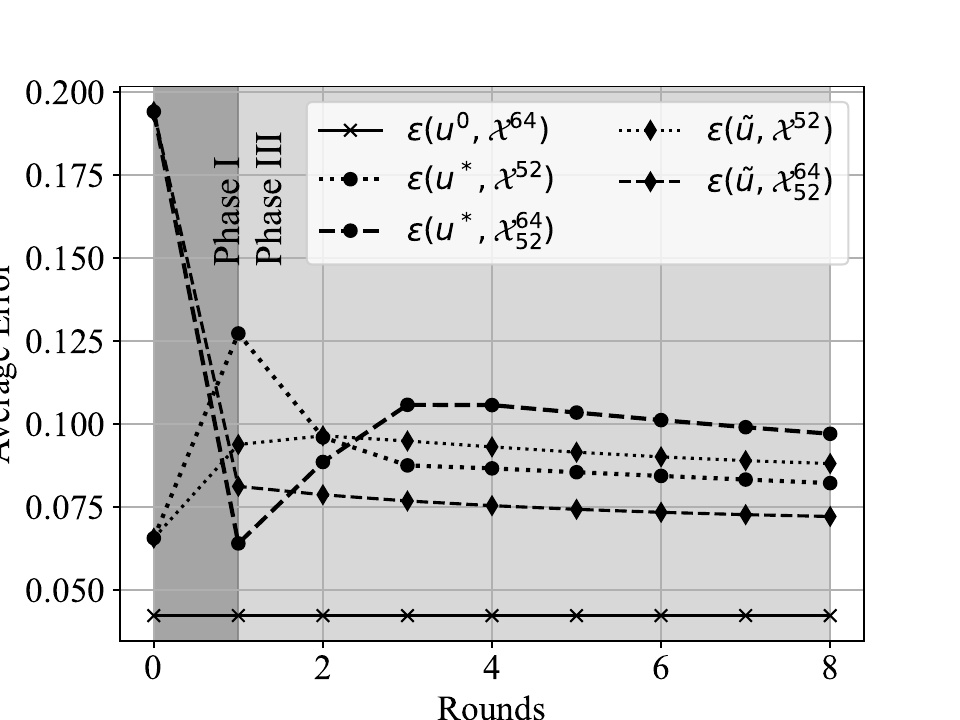}}
\subfigure[]
{\includegraphics[width=0.49\linewidth]{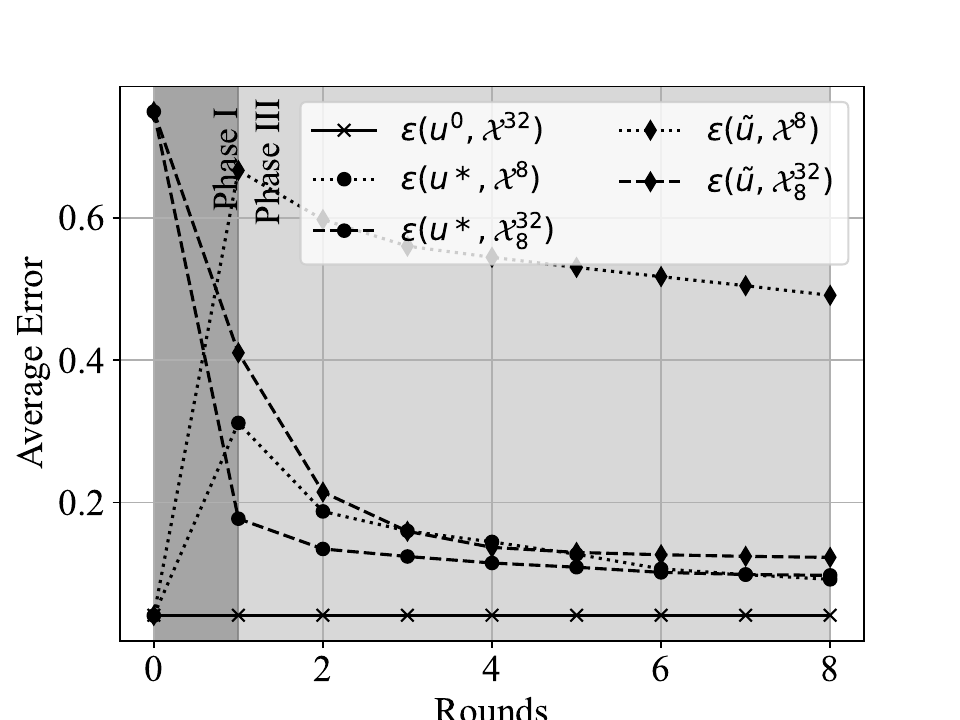}}
\subfigure[]
{\includegraphics[width=0.49\linewidth]{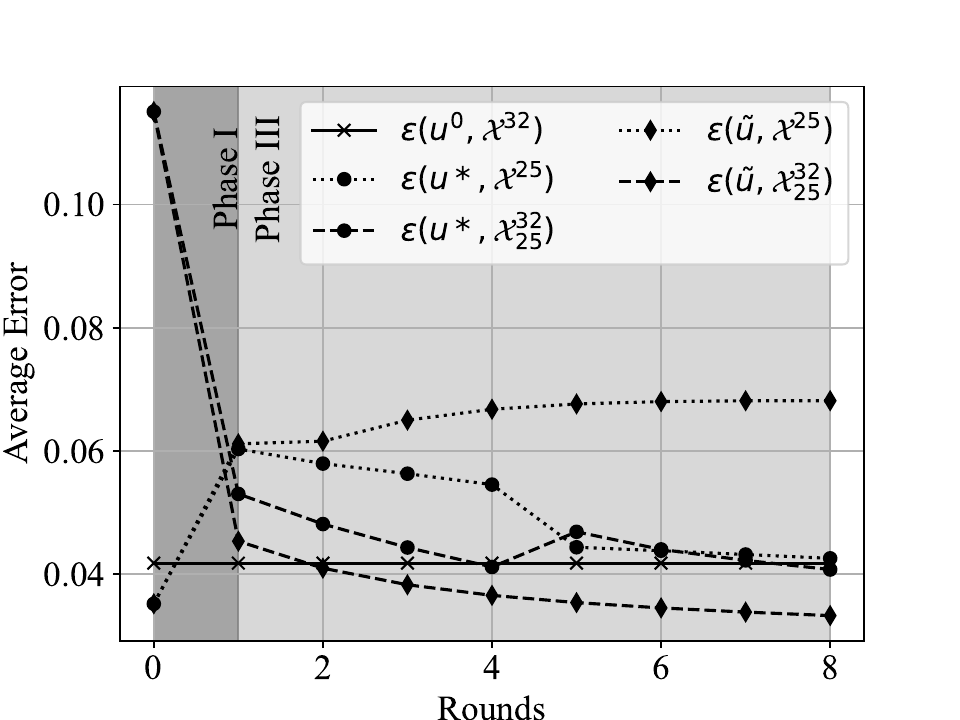}}
\caption{(a) and (b) average error as a function of number of rounds for $|\mathcal{X}|=64$ for $j=16$ and $j=52$, respectively. (c) and (d) average error as a function of number of rounds for $|\mathcal{X}|=32$ for $j=8$ and $j=25$, respectively. The dark gray region is Phase I region and the light gray region is Phase III region (each round is followed by Phase II). Average error on the given set for the control functions $u^*$,$\Tilde{u}$, and $u^0$ are marked by $\bullet,\blacklozenge$, and $\times$, respectively. }\label{fig:results}\vspace{-9mm}
\end{figure}
\vspace{-5mm}
\section{Summary and Outlook}
\vspace{-3mm}
In this work, we have considered a controlled dynamical system to learn a task. We have studied the case when there is an expansion in the training set. The primary issue addressed is the potential loss of model performance on the original task while adapting the model to a new task-specific dataset, which aligns with the catastrophic forgetting. To address this problem, we have introduced a novel fine-tuning method, {\em tuning without forgetting} up to first-order, inspired by control theory. Our work contributes to the scalability of control methods, offering a novel approach to adaptively handle training set expansions. In our numerical results, we have observed that the proposed algorithm effectively handles changes in ensemble cardinalities, preserving previously learned points and adopting new points.

The present work can be extended in several directions. We will relax some of the assumptions for Theorem~\ref{main:approximation} in future work, as described in Theorem~\ref{thm:banach_manifold}. Also, our method has the disadvantage of storing $\mathcal{X}^j$ to compute the set $\mathcal{K}(u, \mathcal{X}^j)$. To address this, one can keep track of the variation in the kernels to approximate the new kernel.

\vspace{-3mm}
\bibliographystyle{ieeetr}
\bibliography{learning.bib}

\begin{thebibliography}{10}

\bibitem{gouk2021distance}
H.~Gouk, T.~M. Hospedales, and M.~Pontil, ``Distance-based regularisation of deep networks for fine-tuning,'' in {\em ICLR}, ICLR, 2021.

\bibitem{agrachev2020control}
A.~Agrachev and A.~Sarychev, ``Control in the spaces of ensembles of points,'' {\em SIAM Journal on Control and Optimization}, vol.~58, no.~3, pp.~1579--1596, 2020.

\bibitem{wang2024comprehensive}
L.~Wang, X.~Zhang, H.~Su, and J.~Zhu, ``A comprehensive survey of continual learning: Theory, method and application,'' {\em IEEE Transactions on Pattern Analysis and Machine Intelligence}, 2024.

\bibitem{lopez2017gradient}
D.~Lopez-Paz and M.~Ranzato, ``Gradient episodic memory for continual learning,'' {\em Advances in neural information processing systems}, vol.~30, 2017.

\bibitem{xuhong2018explicit}
L.~Xuhong, Y.~Grandvalet, and F.~Davoine, ``Explicit inductive bias for transfer learning with convolutional networks,'' in {\em International Conference on Machine Learning}, pp.~2825--2834, PMLR, 2018.

\bibitem{li2018delta}
X.~Li, H.~Xiong, H.~Wang, Y.~Rao, L.~Liu, and J.~Huan, ``Delta: Deep learning transfer using feature map with attention for convolutional networks,'' in {\em International Conference on Learning Representations}, 2018.

\bibitem{tian2023trainable}
J.~Tian, Z.~He, X.~Dai, C.-Y. Ma, Y.-C. Liu, and Z.~Kira, ``Trainable projected gradient method for robust fine-tuning,'' in {\em Proceedings of the IEEE/CVF Conference on Computer Vision and Pattern Recognition}, pp.~7836--7845, 2023.

\bibitem{farajtabar2020orthogonal}
M.~Farajtabar, N.~Azizan, A.~Mott, and A.~Li, ``Orthogonal gradient descent for continual learning,'' in {\em International Conference on Artificial Intelligence and Statistics}, pp.~3762--3773, PMLR, 2020.

\bibitem{wang2021training}
S.~Wang, X.~Li, J.~Sun, and Z.~Xu, ``Training networks in null space of feature covariance for continual learning,'' in {\em Proceedings of the IEEE/CVF conference on Computer Vision and Pattern Recognition}, pp.~184--193, 2021.

\bibitem{chen2018neural}
R.~T. Chen, Y.~Rubanova, J.~Bettencourt, and D.~K. Duvenaud, ``Neural ordinary differential equations,'' {\em Advances in Neural Information Processing Systems}, vol.~31, 2018.

\bibitem{dupont2019augmented}
E.~Dupont, A.~Doucet, and Y.~W. Teh, ``Augmented neural {O}{D}{E}s,'' {\em Advances in Neural Information Processing Systems}, vol.~32, 2019.

\bibitem{lu2018beyond}
Y.~Lu, A.~Zhong, Q.~Li, and B.~Dong, ``Beyond finite layer neural networks: Bridging deep architectures and numerical differential equations,'' in {\em International Conference on Machine Learning}, pp.~3276--3285, PMLR, 2018.

\bibitem{cuchiero2020deep}
C.~Cuchiero, M.~Larsson, and J.~Teichmann, ``Deep neural networks, generic universal interpolation, and controlled odes,'' {\em SIAM Journal on Mathematics of Data Science}, vol.~2, no.~3, pp.~901--919, 2020.

\bibitem{tabuada2022universal}
P.~Tabuada and B.~Gharesifard, ``Universal approximation power of deep residual neural networks through the lens of control,'' {\em IEEE Transactions on Automatic Control}, vol.~68, pp.~2715--2728, 2023.

\bibitem{liberzon2011calculus}
D.~Liberzon, {\em Calculus of Variations and Optimal Control Theory: A Concise Introduction}.
\newblock Princeton university press, 2011.

\bibitem{agrachev2022control}
A.~Agrachev and A.~Sarychev, ``Control on the manifolds of mappings with a view to the deep learning,'' {\em Journal of Dynamical and Control Systems}, vol.~28, no.~4, pp.~989--1008, 2022.

\bibitem{keller1976numerical}
H.~B. Keller, {\em Numerical Solution of Two Point Boundary Value Problems}.
\newblock SIAM, 1976.

\bibitem{brockett2014early}
R.~Brockett, ``The early days of geometric nonlinear control,'' {\em Automatica}, vol.~50, no.~9, pp.~2203--2224, 2014.

\end{thebibliography}

\end{document}